\documentclass[11pt]{article}
\newcommand{\corresponding}{*}
\usepackage[final]{acl}

\usepackage{times}
\usepackage{latexsym}

\usepackage[T1]{fontenc}

\usepackage[utf8]{inputenc}

\usepackage{microtype}

\usepackage{inconsolata}

\usepackage{graphicx}

\usepackage{subcaption}
\usepackage{booktabs} 
\usepackage[table]{xcolor}
\usepackage{amssymb}
\usepackage{mathtools}
\usepackage{amsthm}
\usepackage{multirow} 
\usepackage{arydshln} 

\usepackage{amsmath}
\usepackage{amsfonts}
\usepackage{bm}
\usepackage{algorithm}
\usepackage{algorithmic}
\usepackage{hyperref}

\usepackage[capitalize,noabbrev]{cleveref}

\usepackage{xcolor}
\theoremstyle{plain}
\newtheorem{theorem}{Theorem}[section]

\newtheorem{lemma}[theorem]{Lemma}

\theoremstyle{definition}

\theoremstyle{remark}

\title{Mutual Enhancement Between Global Tokens and Patch Tokens: From Theory to Practice}

\author{
  Xiusheng Huang\textsuperscript{1,2,3},
  Xin Jiang\textsuperscript{3},
  Jun Zhao\textsuperscript{1,2}, 
  Yequan Wang\textsuperscript{3\corresponding},
  Kang Liu\textsuperscript{1,2\corresponding}
 \\
  $^{1}$The Key Laboratory of Cognition and Decision Intelligence for Complex Systems, \\ Institute of Automation, Chinese Academy of Sciences, Beijing, China \\
  $^{2}$School of Artificial Intelligence, University of Chinese Academy of Sciences \\
  $^{3}$Beijing Academy of Artificial Intelligence\\
  \texttt{huangxiusheng2020@ia.ac.cn}, \texttt{jiangxin@baai.ac.cn}\\
    \texttt{tshwangyequan@gmail.com}, 
    \texttt{\{jzhao,kliu\}@nlpr.ia.ac.cn}
}

\begin{document}
\maketitle
\let\thefootnote\relax\footnotetext{\corresponding Corresponding authors}
\begin{abstract}

Accurate and effective discrete image tokenization is crucial for long image sequence processing. However, current methods rigidly compress all content at a fixed rate, ignoring the variable information density of images and leading to either redundancy or information loss. Inspired by information entropy, we propose \textbf{TaTok}, a \textbf{T}heoretically grounded \textbf{a}daptive image \textbf{Tok}enization framework. We rigorously identify two key drawbacks in existing methods: information insufficiency when reconstructing images with patch tokens alone, and information redundancy among patch tokens. To address these, we introduce global tokens that model mutual information across patch tokens, and a \textbf{D}ynamic \textbf{T}oken \textbf{F}iltering (\textbf{DTF}) algorithm based on cumulative conditional entropy to eliminate redundancy. Experiments confirm TaTok's state-of-the-art performance, delivering a 1.3$\times$ gFID improvement and 8.7$\times$ inference speedup. By allocating tokens according to information richness, TaTok enables more compressed yet accurate image tokenization, offering valuable insights for future research.

\end{abstract}

\section{Introduction}
Accurate and effective discrete image tokenization is crucial for processing long image sequences \cite{vqvae,vqvae2,vit-vqgan}. However, the inherent complexity and variable information density of images pose significant bottlenecks for current tokenization methods—these methods rigidly compress all content at a fixed rate, leading to either redundancy or information loss \cite{residual-vq}.


Discrete visual tokenizers typically comprise an encoder, a quantizer, and a decoder~\cite{vqgan}. The encoder splits an image into patch tokens and embeds them into a vector sequence~\cite{vaswani2017attention,raffel2020exploring}, which the quantizer discretizes~\cite{radford2019language,brown2020language} for the decoder to reconstruct the image. Since reconstruction is trivial with unrestricted sequence length (e.g., retaining all RGB pixels)~\cite{zhang2022opt,chatgpt}, a high-performance tokenizer essentially acts as a compressor that preserves core information in compact tokens. Developing efficient, scalable, and theoretically sound tokenizers is thus pivotal for advancing unified multimodal models~\cite{bard}.

\begin{figure}[t]
    \centering
    \includegraphics[width=7.7cm]{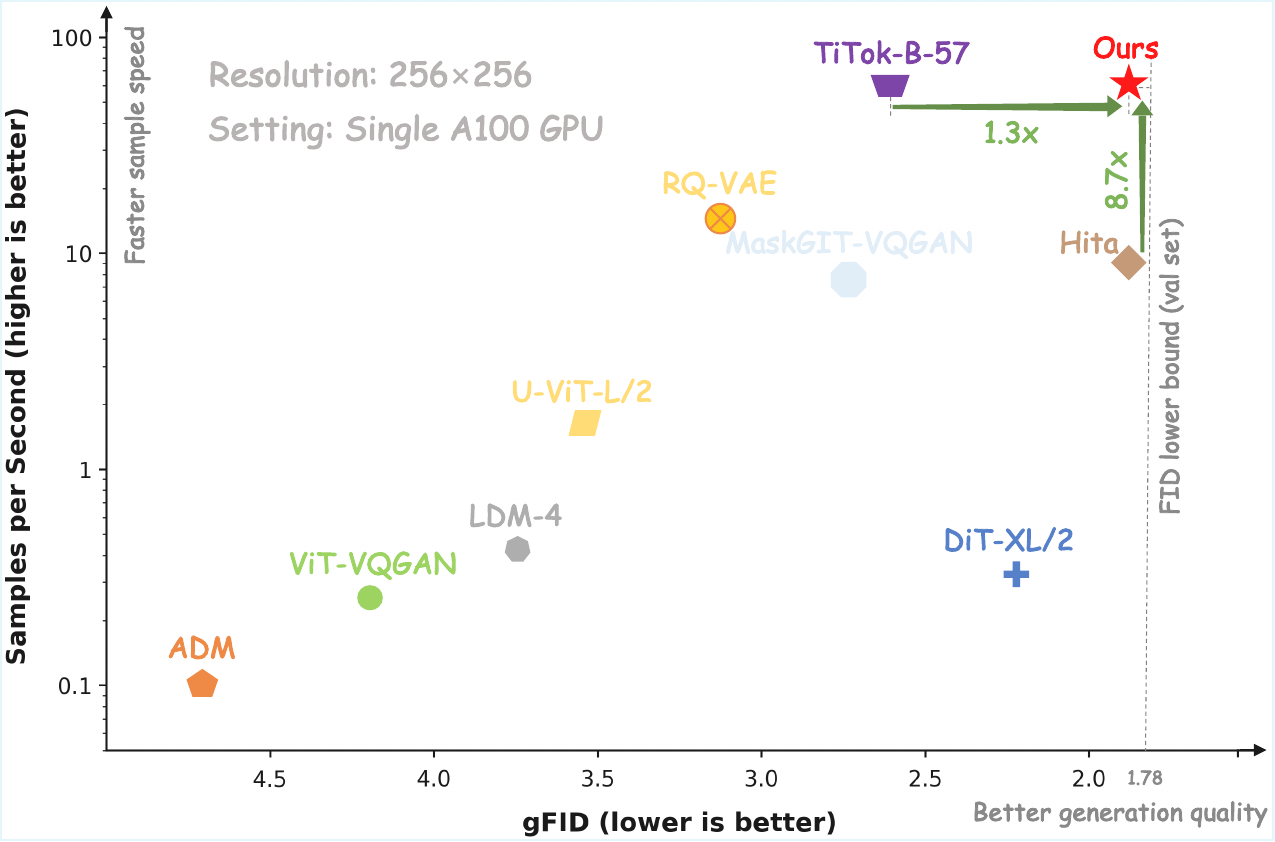}
    \caption{ The comprehensive Speed-Quality Comparison of TaTok and Existing Approaches for ImageNet 256\texttimes256 Image Generation. The sampling speed is measured with an A100 GPU.}
    \label{fig1}
\end{figure}


In real-world scenarios, images carry varying amounts of information due to differences in distribution, texture, and other characteristics~\cite{kaplan2020scaling,henighan2020scaling}, suggesting that the optimal token count should adapt to image-specific features. However, most existing tokenizers apply a fixed compression rate to all images, regardless of architectural or quantization advances. This one-size-fits-all strategy yields redundant tokens for simple images and insufficient tokens for complex ones (causing key information loss), making downstream understanding or generation tasks inefficient or even intractable~\cite{yu2021vector,yu2022scaling}. This is one of the core bottleneck limiting current tokenizers.

To deeply dissect this bottleneck, we conduct theoretical analysis via information entropy theory \cite{gray2011entropy}, formally proving that existing discrete visual tokenization methods have two fundamental defects: (1) \textbf{information insufficiency} (patch tokens alone cannot fully capture image global information) and (2) \textbf{information redundancy} (prevalent among patch tokens). Existing studies attempt to tackle these limitations but lack a principled integrated solution \cite{titok,zheng2025holistic}: hierarchical tokenization methods (e.g., ViT-VQGAN \cite{vit-vqgan}, Swin Transformer \cite{liu2021swin}) capture global context via patch aggregation but raise computational complexity and fail to offset global information loss (Theorem \ref{theorem1} proves the irreducible reconstruction loss lower bound of patch-only schemes), while dynamic token selection approaches (e.g., sparse tokenizers \cite{sfid}) reduce redundancy via heuristic rules (e.g., feature magnitude thresholding) without theoretical guidance on sufficient information for reconstruction, often trading quality for compression. Critically, these methods treat global information modeling and redundancy elimination as separate tasks, failing to synergistically address both defects \cite{maskgit,zheng2025holistic}.

Against this backdrop, inspired by information entropy \cite{gray2011entropy} and rate-distortion \cite{berger2003rate} theory, we propose \textbf{TaTok}, a \textbf{T}heoretically grounded \textbf{a}daptive image tokenization framework. TaTok addresses the dual defects via two complementary mechanisms: \textit{global information completion} to compensate for the inherent insufficiency of patch tokens, and \textit{adaptive redundancy pruning} to dynamically match the number of tokens with the image's information density.


To tackle information insufficiency, we theoretically derive the necessity of introducing a learnable global token that explicitly models the holistic semantic and structural information of images (Lemma \ref{Lemma2.1}). Unlike the auxiliary global features in previous works \cite{titok,zheng2025holistic}, our global token is optimized to minimize conditional entropy, thereby directly enhancing the mutual information between the augmented token sequence and the original image. As formally proven in Theorem \ref{theorem2}, this augmentation not only enriches the information carried by the token sequence but also reduces the theoretical infimum of reconstruction loss, laying a rigorous foundation for high-fidelity image recovery with a limited number of tokens.

For redundancy elimination, we propose a \textbf{D}ynamic \textbf{T}oken \textbf{F}iltering (\textbf{DTF}) algorithm based on cumulative conditional entropy, which adaptively selects the minimum number of patch tokens required for image reconstruction. Guided by rate-distortion theory \cite{berger2003rate}, DTF introduces an information loss rate to balance compression efficiency and reconstruction quality, ensuring that the combined information of global tokens and selected patch tokens meets a predefined minimum information requirement. Specifically, we sort patch tokens by their conditional entropy, a metric quantifying the unique information each token contributes given the global token, and accumulate their information until the dual constraints of sufficient information and an acceptable loss rate are satisfied. This design ensures that only information-rich patch tokens are retained and redundant ones are filtered out, enabling adaptive token allocation that aligns with the inherent information density of each image.


Notably, TaTok is the first tokenizer that unifies global tokens and the dynamic filtering mechanism into an end-to-end trainable framework without introducing additional inference overhead. Extensive experiments demonstrate that TaTok achieves SOTA compression performance. The contributions of this work are summarized as follows:
\begin{itemize}
    \item We theoretically prove that current tokenizers suffer from two critical flaws: global information deficiency and information redundancy in patch tokens.
    \item To address these two flaws, we propose TaTok, which unifies global tokens and the dynamic filtering mechanism into an end-to-end trainable framework, and theoretically proves its superiority.
    \item Extensive experiments show that TaTok achieves superior performance, striking a favorable balance between reconstruction speed and quality.
\end{itemize}

\section{Motivation: Existing inherent flaws}
Focusing on the core limitations of discrete visual tokenizers (based on ViT encoders) in image reconstruction tasks, this section rigorously proves two inherent flaws using information theory and rate-distortion \cite{berger2003rate} theory: information insufficiency and information redundancy.

\subsection{Framework and Core Constraints of Discrete Visual Tokenizers}
We consider a natural image represented as \( x \in \mathcal{X} \subseteq \mathbb{R}^{H \times W \times C} \), where \( H \) and \( W \) denote the height and width of the image, respectively, and \( C \) denotes the number of channels. The probability distribution \( P(x) \) of the image satisfies \textbf{local stationarity} (statistical correlation between any pixel and its neighboring pixels) and \textbf{non-local semantic correlation} (distributional association between different regions of the same semantic object) \cite{vit-vqgan}. A discrete visual tokenizer consists of three deterministic mappings:

\begin{itemize}
    \item  \textbf{Patch Partitioning}: The image \( x \) is partitioned into \( N = \lceil \frac{H - s(1-r)}{s r} \rceil \times \lceil \frac{W - s(1-r)}{s r} \rceil \) patches with size \( s \times s \) and overlap rate \( r \in [0,1) \), denoted as \( \{x_1, x_2, ..., x_N\} \) (where \( x_i \in \mathbb{R}^{s^2C} \));

    \item  \textbf{ViT Encoding}: A global interaction mapping \( \text{Enc}: \mathbb{R}^{N \times s^2C} \to \mathbb{R}^{N \times D} \) (where \( D \) is the token dimension), with the pipeline: “Patch Embedding \( \to \) Positional Encoding \( \to \) Self-Attention \( \to \) FFN", outputting a sequence of continuous tokens \( z = \text{Enc}(x) \);

    \item \textbf{Quantization Discretization}: A quantization function \( Q: \mathbb{R}^{N \times D} \to \{1,2,...,K\}^N \) (where \( K \) is the codebook size), outputting discrete token indices \( q = Q(z) \). The reconstructed image \( \hat{x} = \text{Dec}(q) \) is finally obtained via the decoder \( \text{Dec}: \{1,...,K\}^N \to \mathbb{R}^{H \times W \times C} \).
\end{itemize}

The core goal of tokenization is to minimize the reconstruction loss \( \mathcal{L}(x, \hat{x}) = \mathbb{E}[\|x - \hat{x}\|_2^2] \) under a finite code rate \( R = \frac{N \log_2 K}{H W C} \) (bits per pixel).

\subsection{Defect I: Information Insufficiency}\label{sec3.2}

The essence of information insufficiency is that the irreversible mappings in the tokenization pipeline are constrained by rate-distortion \cite{berger2003rate} theory, making zero-distortion reconstruction impossible under finite code rates. First, both patch partitioning and ViT encoding are deterministic mappings. By the Data Processing Inequality \cite{beaudry2011intuitive}, the differential mutual information of continuous variables satisfies \( I(x; z) \leq I(x; \{x_1,...,x_N\}) \). Since patch partitioning is a full-coverage mapping (\( x = \text{Concatenate}(x_1,...,x_N) \)), \( I(x; \{x_1,...,x_N\}) = h(x) \) (where \( h(x) \) is the differential entropy of \( x \)), so \( I(x; z) \leq h(x) \). When \( D < s^2C \), the dimension reduction in patch embedding is a many-to-one mapping—multiple \( \{x_1,...,x_N\} \) correspond to the same \( z \)—thus \( I(x; z) < h(x) \), and information loss occurs during encoding.

Further, the quantization operation \( Q(z) = q \) is a deterministic interval-partitioning mapping. The mutual information between continuous and discrete variables satisfies \( I(x; q) = I(x; z) - I(x; z \mid q) \). Since \( I(x; z \mid q) \geq 0 \), we have \( I(x; q) \leq I(x; z) < h(x) \), and quantization exacerbates information loss. According to \textbf{Shannon's Rate-Distortion Theorem} \cite{berger2003rate}, for any continuous source, the rate-distortion function satisfies \( R(D) \geq h(x) - \frac{1}{2} \log_2(2\pi e D) \). Specifically, for a Gaussian source, \( R(D) = \frac{1}{2} \log_2 \left( \frac{\sigma_x^2}{D} \right) \) (when \( D \leq \sigma_x^2 \), where \( \sigma_x^2 = \text{Var}(x) \) is the pixel variance). Natural images are not Gaussian distributions, and their rate-distortion functions have more complex forms, but they still satisfy similar lower-bound constraints—a finite code rate \( R \) corresponds to a positive minimum distortion \( D_{\text{min}} > 0 \).

\begin{theorem}[Reconstruction Error Infimum of Discrete Visual Tokenizers]\label{theorem1}
For a discrete visual tokenizer consisting of an encoder $\text{Enc}$, a quantizer $Q$, and a decoder $\text{Dec}$, the theoretical infimum of the reconstruction loss $\mathcal{L}$ (measuring discrepancy between original image $x$ and reconstructed image $\text{Dec}(Q(\text{Enc}(x)))$) satisfies:
\begin{equation}
\begin{aligned}
    \epsilon_{\text{inf}} = &\inf_{\text{Enc},Q,\text{Dec}} \mathcal{L}\left(x, \text{Dec}\left(Q\left(\text{Enc}(x)\right)\right)\right) \\ &\geq D_{\text{min}} > 0
\end{aligned}
\end{equation}
\end{theorem}

where $D_{\text{min}}$ denotes a non-zero positive constant. This implies that high-frequency details and continuous grayscale information of the original image \textbf{\textit{cannot be fully recovered using only patch tokens}}. We present corresponding verification experiments to validate this in Experiment \ref{exI}.

Note that this conclusion relies on the standard design of discrete visual tokenizers (finite code rate, non-invertible encoding). While zero-distortion reconstruction is theoretically feasible with invertible encoders (e.g., Flow-based Transformers \ {ma2024sit}) and infinite code rates ($K \to \infty$), this contradicts the core design goals of discrete tokenizers (e.g., “low code rate, discrete representation") and is thus impractical in real-world scenarios.

\begin{figure*}[t]
    \centering
    \includegraphics[width=15.0cm]{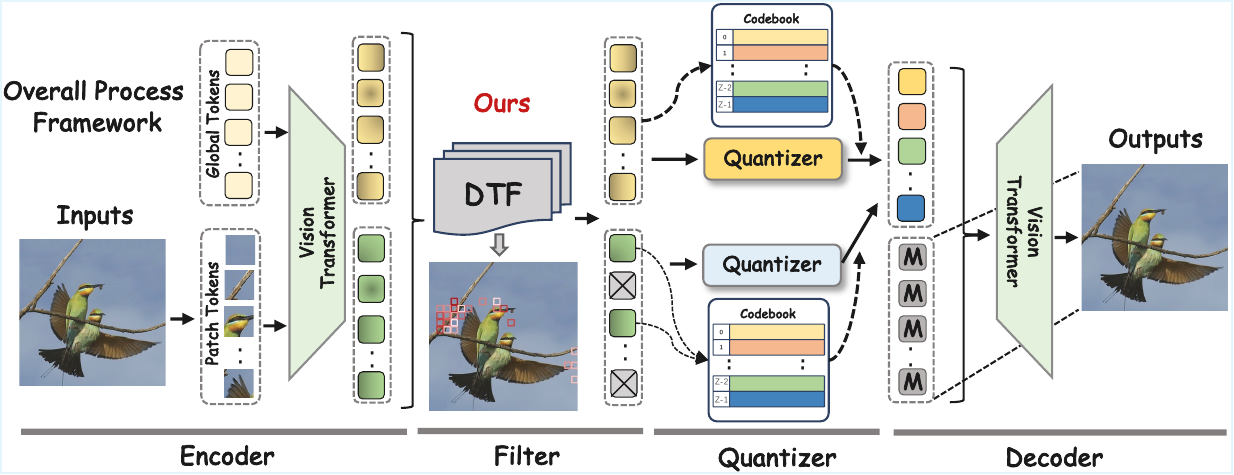}
    \caption{ The overall process framework of TaTok. The input image is encoded into patch tokens via a vision transformer encoder, and we additionally introduce learnable global tokens. Our proposed DTF module then removes redundant patch tokens, which are further quantized by dual quantizers with codebooks. Finally, the decoder reconstructs the output image from the processed tokens.}
    \label{fig2}
\end{figure*}

\subsection{Defect II: Information Redundancy}\label{DefectII}
Based on multivariate mutual information \cite{lee2013feature}, the redundancy of a continuous token sequence \( z = [z_1,...,z_N] \) is defined as:
\begin{equation}
\begin{aligned}
    \text{Red}(z) = \frac{\sum_{i=1}^N h(z_i) - h(z)}{\sum_{i=1}^N h(z_i)}
    \label{eq1}
\end{aligned}
\end{equation}
where the numerator is the total mutual information between tokens (quantifying information overlap), and the denominator is the sum of the differential entropies of individual tokens (quantifying the total information of tokens). It should be noted that the differential entropy \( h(\cdot) \) of continuous variables can be negative, but this only affects the absolute value of the entropy and does not change the quantitative meaning of redundancy: when tokens are fully independent, \( h(z) = \sum_{i=1}^N h(z_i) \) and \( \text{Red}(z) = 0 \); when tokens are fully correlated, \( h(z) \approx h(z_1) \) (information of other tokens is completely contained in the first token) and redundancy approaches 1.

The local stationarity of natural images implies that adjacent patches \( x_i, x_j \) satisfy \( I(x_i; x_j) = h(x_i) + h(x_j) - h(x_i,x_j) > 0 \); non-local semantic correlation means non-adjacent patches \( x_k, x_l \) also satisfy \( I(x_k; x_l) > 0 \). The global interaction mechanism of ViT encoders typically preserves or enhances correlations between patches; thus, for relevant patch pairs where the original patches satisfy \( I(x_i; x_j) > 0 \), their corresponding tokens usually also satisfy \( I(z_i; z_j) > 0 \). Accordingly, the total mutual information between tokens \( \sum_{i=2}^N I(z_i; z_1,...,z_{i-1}) \geq \sum_{\text{relevant pairs}} I(z_i; z_j) > 0 \), so substituting into the redundancy definition gives \( \text{Red}(z) > 0 \); for discrete tokens \( q = Q(z) \), we similarly obtain \( \text{Red}(q) > 0 \). This indicates that \textbf{\textit{there is information redundancy between patch tokens}}. We present corresponding verification experiments to validate this in Experiment \ref{exII}. 

Redundancy has two negative impacts: (1) Increased Computational Overhead, the self-attention complexity of ViT is \( O(N^2D) \); redundancy prevents excessive reduction of \( N \), leading to complexity growing with \( N^2 \); (2) Degraded Reconstruction Quality, superposition of redundant information in tokens causes blocking artifacts (non-overlapping partitioning) or over-smoothing (overlapping partitioning) in reconstructed images, while amplifying quantization noise.

Overall, the information insufficiency and redundancy flaws of discrete visual tokenizers originate from the inherent contradiction between their discretization design and the continuous correlation of natural images: information insufficiency stems from rate-distortion constraints of three irreversible mappings, while redundancy stems from the inherent correlation of natural images and information fusion via global interaction. These two defects have not been completely eliminated in previous work, but they inspired us to design a new tokenizer framework to eliminate them.


\section{TaTok: Theoretically Grounded Adaptive Tokenizer}

To address the dual defects of \textbf{information insufficiency} and \textbf{information redundancy} in discrete visual tokenizers, we theoretically derive the necessity of introducing a learnable global token $g \in \mathbb{R}^D$ to augment the token sequence, and proposes an adaptive patch token filtering module to eliminate redundancy.

\subsection{Necessity of Global Token: Information Completion under Practical Constraints}\label{sec3_1}

According to Theorem \ref{theorem1}, the core of the information insufficiency problem lies in the fact that patch tokens cannot fully capture the global information $\mathcal{G}(x)$, resulting in a strict upper bound on mutual information $I(x; z) < h(x)$. Theoretical deduction shows that introducing global tokens has an enhancing effect on mutual information, and combined with rate distortion theory, it is proven that the global token enhancement mechanism can reduce the theoretical lower bound of reconstruction loss.

\begin{lemma}[Mutual Information Enhancement Effect of Global Token]\label{Lemma2.1}

For the augmented token sequence $z' = [g; z]$ generated by the global token-augmented encoder $\text{Enc}'$, if $g$ is optimized to approximate the global information $\mathcal{G}(x)$ (e.g., $\lim_{\text{training steps} \to \infty} h(\mathcal{G}(x) \mid g) = 0$), the mutual information between the augmented token sequence and the original image satisfies:
\begin{equation}
\begin{aligned}
    I(x; z') = I(x; z) + I(x; g \mid z)
    \label{eq2}
\end{aligned}
\end{equation}
\end{lemma}

where $I(x; g \mid z) \approx I(x; \mathcal{G}(x) \mid z) > 0$. This lemma indicates that the global token augmentation mechanism enables the augmented token sequence to carry more image information, directly compensating for the information loss caused by patch partitioning, dimension reduction, and quantization operations. We provide detailed theoretical proof in the Appendix \ref{Proof1}.

\begin{theorem}[Reconstruction Loss Reduction via Global Token]\label{theorem2}
For a discrete visual tokenizer augmented with a learnable global token (consisting of encoder $\text{Enc}'$, quantizer $Q$, and decoder $\text{Dec}'$), the theoretical infimum of its reconstruction loss satisfies:
\begin{equation}
\begin{aligned}
\epsilon'_{\text{inf}} = \inf_{\text{Enc}', Q, \text{Dec}'} \mathcal{L}\left(x, \text{Dec}'\left(Q\left(\text{Enc}'(x)\right)\right)\right) < \epsilon_{\text{inf}}
\end{aligned}
\end{equation}
\end{theorem}

where $\epsilon_{\text{inf}}$ is the reconstruction loss infimum of the patch-token-only scheme described in Theorem \ref{theorem1}. We provide detailed theoretical proof in the Appendix \ref{Proof2}.

Notably, the global token scheme does not conflict with the core design goals of discrete visual tokenizers (low code rate, discrete representation, and efficient inference). Compared with other global information modeling methods (e.g., hierarchical patch aggregation, sparse attention), the global token scheme achieves a better trade-off among information modeling capability, computational complexity, and training stability under practical constraints. This provides a rigorous theoretical foundation for the subsequent design of improved tokenizer frameworks.

\subsection{Dynamic Token Filtering Algorithm}\label{sec3_2}

In order to eliminate the redundancy of patch tokens, especially under the condition of global tokens, we introduce a \textbf{D}ynamic \textbf{T}oken \textbf{F}iltering (\textbf{DTF}) algorithm based on conditional entropy. Consistent with previous research, the entire tokenizer completes the image reconstruction task during the pre training process. Based on rate-distortion \cite{berger2003rate} theory, DTF finds the minimum number of Patch Tokens \( N \) such that the total information of “Global Tokens + \( N \) Patch Tokens" meets the minimum information requirement for image reconstruction tasks.

\subsubsection{Core Constraint: Information Loss Rate}
To balance ``minimum number of Tokens'' and ``acceptable distortion'', introduce the information loss rate $\epsilon$ ($\epsilon \in [0,1)$), which indicates that the proportion of unique information lost after selection does not exceed $\epsilon$. Meanwhile, we must satisfy Eq.\ref{eq12}. Therefore, $N$ needs to satisfy the following dual constraints: (1) Minimum Information Constraint: $H_N \geq C$. (2)Information Loss Rate Constraint: $H_N \geq (1 - \epsilon) H_{\text{total}}$ (controls the proportion of information loss). Combining these constraints:
\begin{equation}
\begin{aligned}
    H_N \geq \max\{ T,\ (1 - \epsilon) H_{\text{total}} \}
    \label{eq25}
\end{aligned}
\end{equation}

where $T$ represents the minimum unique information that needs to be supplemented by patch tokens, detailed definitions are provided in Appendix \ref{Proof3}.

\subsubsection{Definition and Existence of $N$}
The theoretical information lower bound $N$ is the smallest positive integer satisfying:
\begin{equation}
\begin{aligned}
    N = \min \{ &k \in \mathbb{N}^+ \mid \sum_{i=1}^k H(\hat{p}_{\sigma(i)} \mid \hat{G}) 
    \geq \\
    &\max\{ T,\ (1 - \epsilon) H_{\text{total}} \} \}
    \label{eq25}
\end{aligned}
\end{equation}

We provide detailed theoretical proof in the Appendix \ref{Proof3}.

\subsection{TaTok De-Tokenization: Reconstruction via Augmented Token Sequence}

To implement TaTok’s image reconstruction, we design a de-tokenization pipeline using an augmented token sequence (replacing conventional 1D latent tokens).

\subsubsection{Augmented Latent Token Sequence}
Recall DTF: select top-\(N\) patch tokens (sorted by \(H(\mathbf{z}_i \mid \mathbf{G})\)) to form \( \mathbf{z}_\tau \in \mathbb{R}^{N \times D} \). We use \(k\) learnable global tokens \( \mathbf{G} = [\mathbf{g}_1, \dots, \mathbf{g}_k] \in \mathbb{R}^{k \times D} \) (for multi-scale global info) and concatenate them with \( \mathbf{z}_\tau \):
\begin{equation}
\begin{aligned}
\mathbf{z}_{\text{aug}} = \mathbf{G} \oplus \mathbf{z}_\tau \in \mathbb{R}^{(k+N) \times D}
\end{aligned}
\end{equation}
where \( \oplus \) denotes token-dimension concatenation.

\subsubsection{Quantization and Decoder Input}
Apply quantization \( \text{Quant}: \mathbb{R}^{(k+N) \times D} \to \{1,\dots,K\}^{(k+N)} \) (matching the \(Q(\cdot)\) of Theorem \ref{theorem1}) to get discrete tokens \( \mathbf{q}_{\text{aug}} = \text{Quant}(\mathbf{z}_{\text{aug}}) \). To align with initial patch count \(N_0\), complement \( \mathbf{q}_{\text{aug}} \) with \(N_0 - (k+N)\) replicated learnable mask tokens \( \mathbf{M} \), forming the decoder input:
\begin{equation}
\begin{aligned}
\mathbf{S} = \mathbf{q}_{\text{aug}} \oplus \mathbf{M}
\end{aligned}
\end{equation}
The extended ViT decoder outputs the reconstructed image:
\begin{equation}
\begin{aligned}
\hat{\mathbf{x}} = \text{Dec}'(\mathbf{S})
\end{aligned}
\end{equation}

\subsubsection{Training Loss}
The total loss $\mathcal{L}_{\text{total}}$ jointly optimizes tokenization (encoder+DTF+global tokens) and de-tokenization (decoder), comprising three weighted components:
1) Reconstruction Loss: $\mathcal{L}_{\text{rec}} = \mathbb{E}_{\mathbf{x}} [\|\mathbf{x} - \hat{\mathbf{x}}\|_2^2]$ (pixel-wise discrepancy).
2) Quantization Loss: $\mathcal{L}_{\text{commit}} = \mathbb{E}_{\mathbf{x}} \left\| \mathbf{z}_{\text{aug}} - \mathbf{q}_{\text{aug}} \right\|_2^2$, where $\mathbf{z}_{\text{aug}} = [g; \mathbf{z}]$ (augmented token sequence) and $\mathbf{q}_{\text{aug}}$ is the quantized codebook $\mathbf{E}$ output, aligning augmented tokens with the codebook.
3) Global Semantic Alignment Loss: $\mathcal{L}_{\text{glob}} = \mathbb{E}_{\mathbf{x}} \left[ h(\mathcal{G}(\mathbf{x}) \mid g) \right]$, minimizing conditional entropy between image holistic info $\mathcal{G}(\mathbf{x})$ and global token $g$ (Lemma \ref{Lemma2.1}, Theorem \ref{theorem2}). The total loss is defined as:
\[
\mathcal{L}_{\text{total}} = \mathcal{L}_{\text{rec}} + \lambda_1 \mathcal{L}_{\text{commit}} + \lambda_2 \mathcal{L}_{\text{glob}}
\]
where $\lambda_1, \lambda_2 > 0$ are balancing hyperparameters, and both are set to 1 throughout all experiments.

\begin{table*}[!t]
\centering
\small
\setlength{\tabcolsep}{4.0pt}    
\renewcommand{\arraystretch}{1.1} 
\caption{\textbf{ImageNet-1K $256 \times 256$ generation results evaluated.} $\dagger$: Trained on OpenImages \cite{openimage}; $\ddagger$: Trained on OpenImages, LAION-Aesthetics/-Humans \cite{schuhmann2022laion}. P: generator's parameters. S: sampling steps. T: throughput as samples per seconds on A100 with float32 precision.}\label{table1}
\begin{tabular}{lccccc|lcccc}
\hline 
\multicolumn{6}{c|}{\textbf{tokenizer}} & \multicolumn{5}{c}{\textbf{generator}} \\
\hline 
Model & Ratio $\uparrow$ & \#Tokens & Codebook Size & rFID $\downarrow$ & - & Model & gFID $\downarrow$ & P $\downarrow$ & S $\downarrow$ & T $\uparrow$ \\
\hline
\multicolumn{10}{c}{\textit{diffusion-based generative models}} \\
\hline
Taming-VQGAN$^\dagger$ & 1 & 1024 & 16384 & 1.14 & & LDM-8  & 7.76 & 258M & 200 & - \\
VAE$^\dagger$ & 1 & 4096$\times$3 & - & 0.27 & & LDM-4 & 3.60 & 400M & 250 & 0.4 \\ \hdashline 
\multirow{3}{*}{VAE $^\ddagger$} & \multirow{3}{*}{1} & \multirow{3}{*}{1024$\times$4} & \multirow{3}{*}{-} & \multirow{3}{*}{0.62} & & UViT-L/2  & 3.40 & 287M & 50 & 1.1 \\
& & & & & & UViT-H/2  & 2.29 & 501M & 50 & 0.6 \\
& & & & & & DiT-XL/2  & 2.27 & 675M & 250 & 0.6 \\
\hline
\multicolumn{10}{c}{\textit{transformer-based generative models}} \\
\hline
Taming-VQGAN  & 1 & 256 & 1024 & 7.94 & & Taming-Transformer  & 15.78 & 1.4B & 256 & 7.5 \\
RQ-VAE  & 1 & 256 & 16384 & 3.20 & & RQ-Transformer  & 8.71 & 1.4B & 64 & 16.1 \\
ViT-VQGAN  & 1 & 1024 & 8192 & 1.28 & & VIM-Large  & 4.17 & 1.7B & 1024 & 0.3 \\
Hita  & 1 & 569 & 16384 & 1.50 & & MaskGIT-ViT  & 1.89 & 177M & 8 & 10.8 \\
\hdashline 

SFTok-B-64 & 4.0x & 64 & 8192 & 1.44 & & MaskGIT-ViT  & 2.06 & 177M & 8 & 89.8 \\
SFTok-B-57 & 4.4x & 57 & 8192 & 1.53 & & MaskGIT-ViT  & 2.21 & 177M & 8 & 94.0 \\
TiTok-B-64 & 4.0x & 64 & 4096 & 1.70 & & MaskGIT-ViT  & 2.48 & 177M & 8 & 89.8 \\
TiTok-B-57 & 4.4x & 57 & 4096 & 1.75 & & MaskGIT-ViT  & 2.57 & 177M & 8 & 94.0 \\

\hdashline 
\rowcolor{gray!20}\textbf{TaTok(Ours)} & \textbf{4.5x} & \textbf{56.4} & 4096 & 1.51 & & MaskGIT-ViT  & \textbf{1.89} & 177M & 8 & \textbf{94.3} \\
\hline 
\end{tabular}
\end{table*}

\section{Experiments}
\subsection{Setup}
Following prior works, we set the codebook dimension for patch-level and global image tokens to 8 and 12, respectively. As discussed in , this configuration balances reconstruction quality and efficient codebook utilization. Both codebooks are fixed to a size of 4096, and the number of queries for capturing global features is set to 16.  More details are provided in Appendix \ref{exp_setup}.

\subsection{Experiment I: Global Information is Missing}\label{exI}

To verify global information deficiency in patch-only discrete visual tokenizers (Defect \hyperref[sec3.2]{I}), we performed two rFID-based controlled experiments (balanced total tokens, Fig.\ref{fig_ex1}). Setting 1 (fixed patch tokens=441, incremental global tokens): initial rFID=2.17, dropping to 1.51 (30.4\% reduction) at 16 global tokens and plateauing at 1.48 (200 tokens). Setting 2 (fixed total tokens=441, global tokens replacing patch tokens): initial rFID=2.17, falling to 1.54 at 16 global tokens (higher than Setting 1), 1.50 at 65 tokens, and 1.52 at 200 tokens (all $>$ Setting 1).

Key insights: (1) Reconstruction gains stem from global information supplementation, not more total tokens; (2) Excessive patch token replacement fails to sustain improvements (patch tokens carry essential local information). Experiments confirm severe global deficiency in patch-only tokenizers: 16 global tokens maximize reconstruction quality, while excess tokens offer no benefits and may slightly degrade quality.

\subsection{Experiment II: Patch Tokens Contain Redundant Information}\label{exII}

To verify information redundancy in patch-only discrete visual tokenizers (Defect \hyperref[DefectII]{II}), we conducted three controlled experiments (start/end/random truncation, global tokens=0, rFID as core metric, Fig.\ref{fig_ex2}). All groups started with 441 intact patch tokens (rFID=2.17), with rFID gently fluctuating early and surging exponentially later. Deleting 0$\rightarrow$211 tokens (230 remaining) caused $<$11\% rFID growth: end truncation (2.32, slowest), start truncation (2.35, intermediate), random truncation (2.40, fastest). Beyond 211 deletions, rFID soared exponentially; deleting 440 tokens led to rFID$=$379 (complete reconstruction collapse).

Consistent trends confirm: ~50\% patch tokens are redundant (removal has negligible impact), while over 50\% truncation triggers quality collapse (core information loss). Positional heterogeneity (minimal end truncation impact, maximal random) is observed, verifying significant redundancy in patch-only tokenizers and providing basis for precise redundant token screening in optimization.

\begin{figure}[t]
    \centering
    \includegraphics[width=7.7cm]{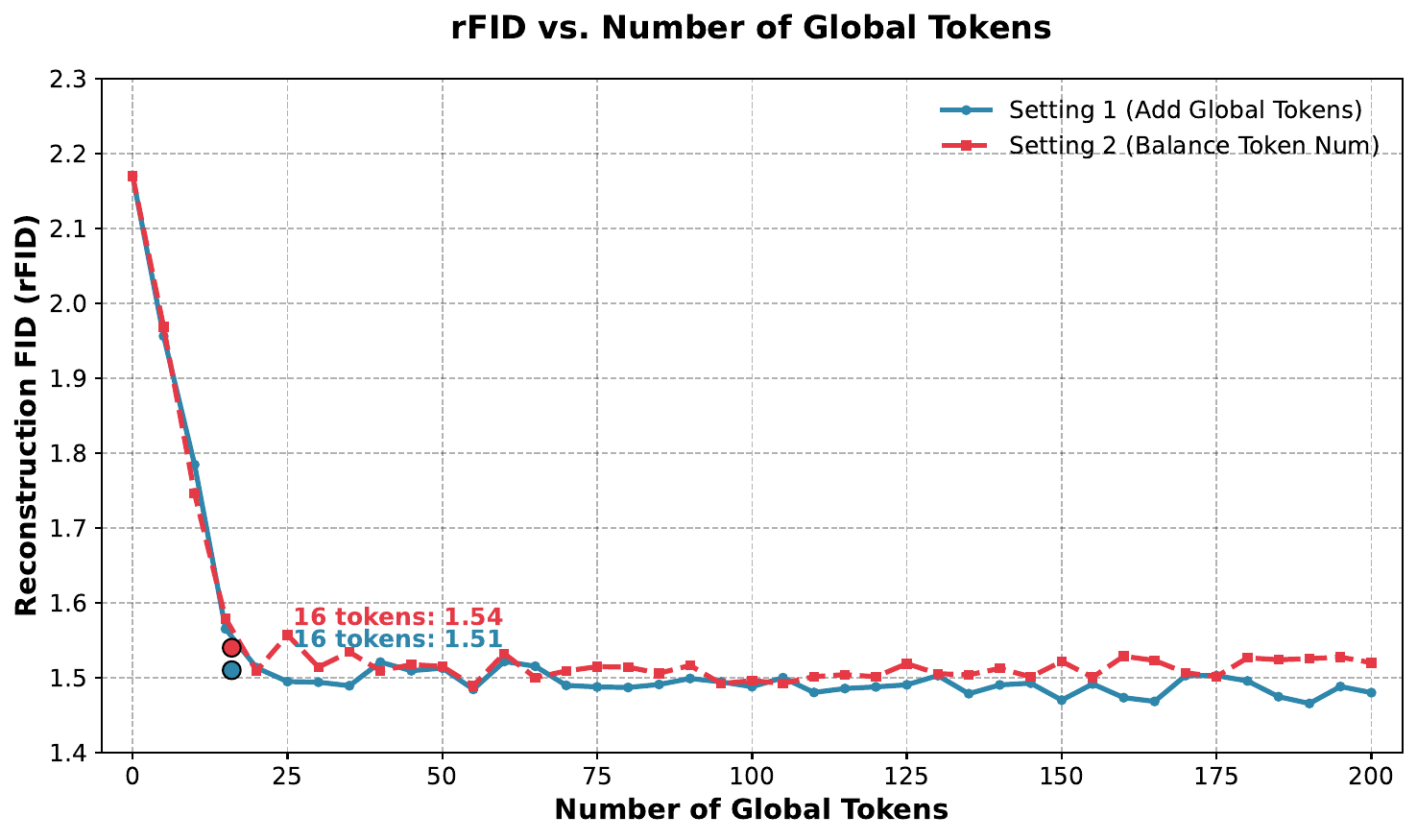}
    \caption{ The plot of experimental results for reconstruction performance versus the number of added global tokens. The x-axis denotes the number of added global tokens, and the y-axis represents rFID.}
    \label{fig_ex1}
\end{figure}

\begin{figure}[t]
    \centering
    \includegraphics[width=7.7cm]{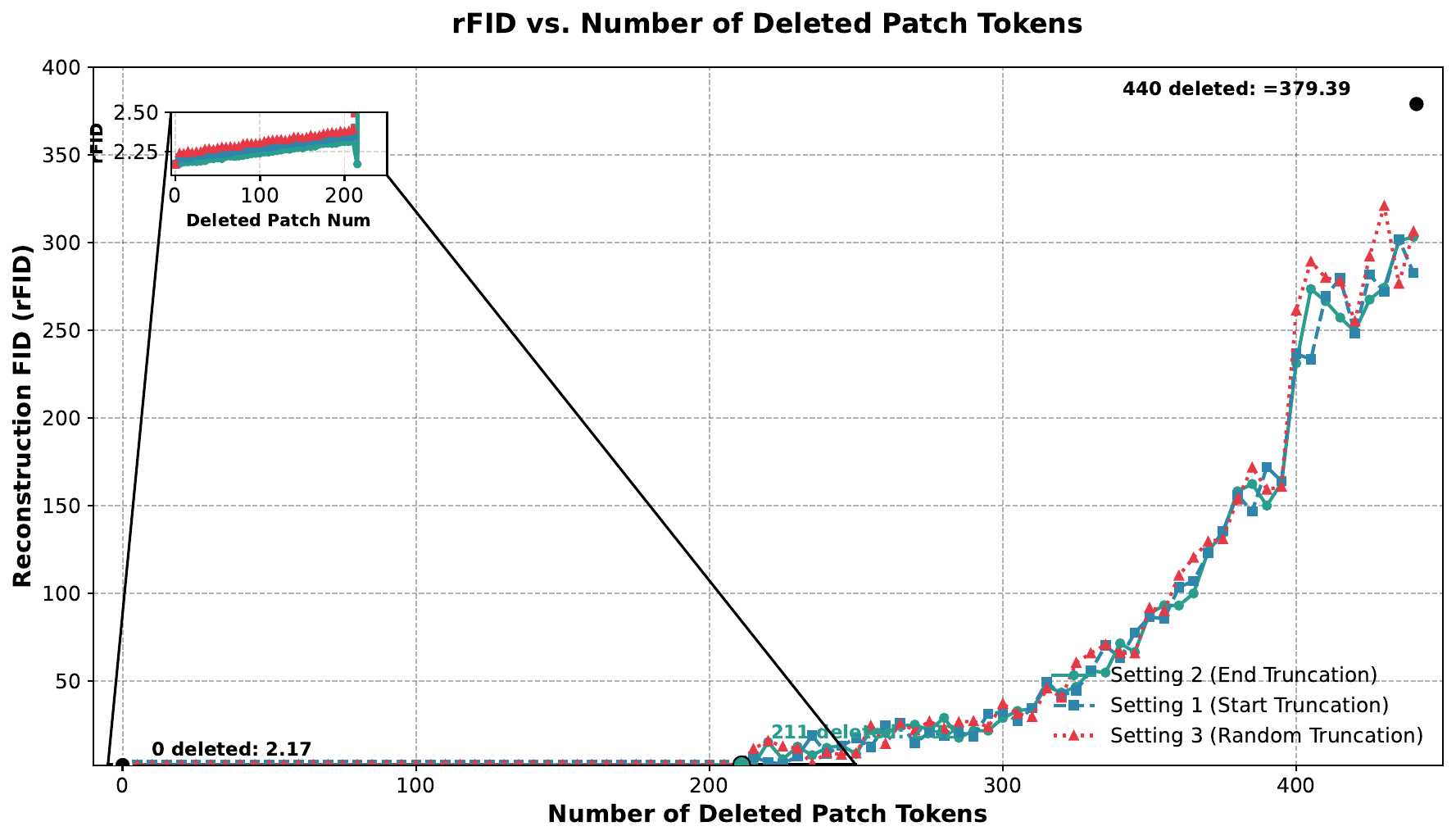}
    \caption{ The experimental results of reconstruction performance versus the number of removed patch tokens. The x-axis represents the number of removed patch tokens, and the y-axis represents the rFID score.}
    \label{fig_ex2}
\end{figure}

\subsection{Experiment III: Comparative Results}

For ImageNet \(256 \times 256\) (Table \ref{table1}), TaTok outperforms other Transformer-based VQ tokenizers in rFID with fewer latent tokens. With 56.4 tokens (codebook size=4096), its rFID=1.51—surpassing MaskGIT-VQGAN (256 tokens, rFID=2.28) \cite{vit-vqgan} and TiTok-B-57 (57 tokens, rFID=1.75), proving superior compact token information preservation.

Paired with MaskGIT-ViT \cite{maskgit}, TaTok also boosts generation FID (gFID=1.89), far outperforming MaskGIT-VQGAN (gFID=6.18) and TiTok-B-64 (gFID=2.48) with the same generator, validating more efficient generator training.

In efficiency, TaTok is lightweight: paired MaskGIT-ViT has 177M parameters, 8 sampling steps, and 94.3 samples/s throughput on A100 (float32). Vs. MaskGIT-VQGAN+MaskGIT-ViT (9.7 samples/s, 3.8B parameters), TaTok achieves 9.7x sampling speedup and drastically reduced parameters, balancing quality and efficiency better.

\begin{figure}[t]
    \centering
    \includegraphics[width=6.5cm]{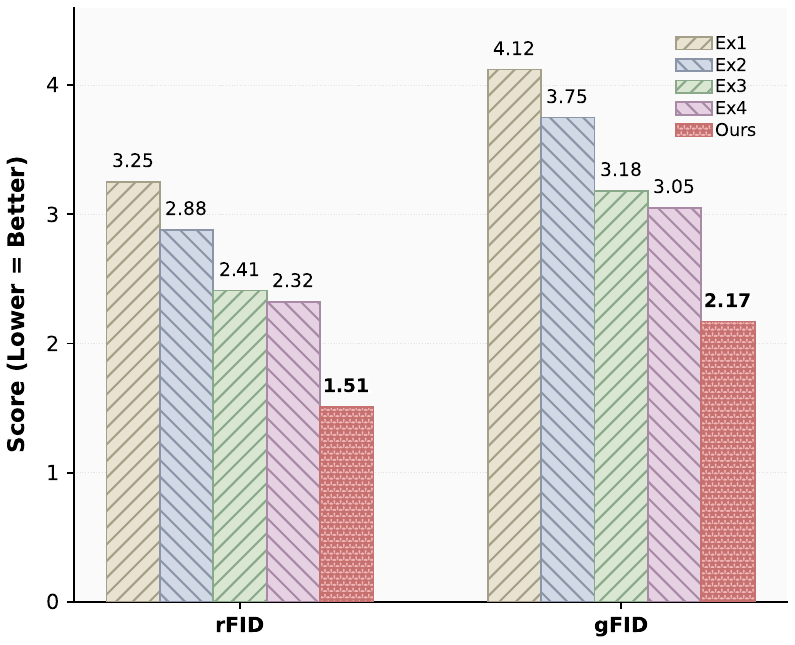}
    \caption{ The ablation study results comparing our TaTok with other experimental settings.}
    \label{table2}
\end{figure}

\begin{figure}[t]
    \centering
    \includegraphics[width=7.7cm]{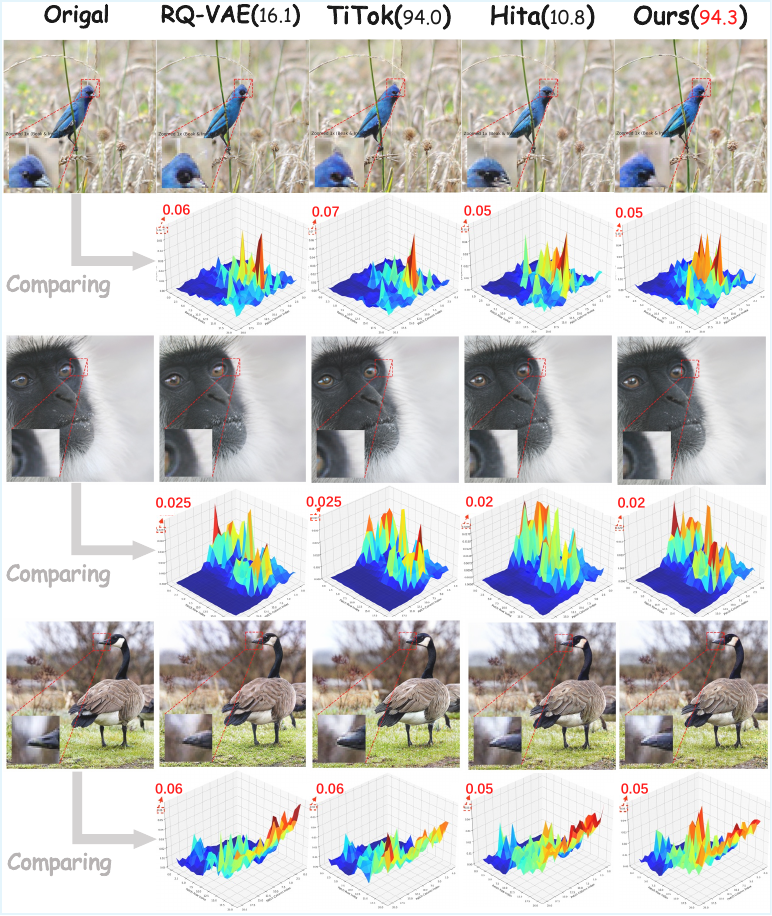}
    \caption{ Visualization results of ablation experiments. The numbers in parentheses denote the sampling speed.}
    \label{fig_ex5}
\end{figure}

\subsection{Experiment IV: Ablation Study}

To validate the rationality of our top-N patch token selection algorithm, we conducted ablation experiments with four baselines for comparison: Random Sampling (Ex1), Uniform Sampling (Ex2), First N Tokens (Ex3), Last N Tokens (Ex4). MaskGIT-ViT was used as the generation model, with results in Fig.\ref{table2}.

Baselines show hierarchical rFID/gFID performance: Ex1 (worst, rFID=3.25, gFID=4.12), Ex2 (rFID=2.88, gFID=3.75), Ex3 (rFID=2.41, gFID=3.18) and Ex4 (rFID=2.32, gFID=3.05) with better performance. Our method outperforms all baselines significantly, achieving optimal rFID=1.51 (34.9\% lower than Ex4) and gFID=2.17 (28.8\% lower than Ex4), validating our strategy’s rationality in filtering redundancy and retaining core local information to boost reconstruction/generation quality.

Visualization results (Fig.\ref{fig_ex5}) further confirm our strategy’s superiority: for bird/monkey/goose images, our reconstructions match originals best in fine-grained details (e.g., fur texture, feather/body outline) and semantic consistency. Feature heatmaps show our feature space aligns best with originals, outperforming RQ-VAE \cite{vqvae2}, TiTok \cite{titok} and Hita \cite{zheng2025holistic} by preserving core local information and reducing redundant token interference. Additionally, Appendix \ref{Discussion} finds edge patch tokens carry more positional information, providing novel token interpretability insights.

\section{Related Work}
Discrete image tokenization has evolved from autoencoders and VAEs~\cite{prec_recall} to VQ-VAEs~\cite{vqvae2,is}, VQGAN~\cite{vqgan,magvit2}, and lookup-free variants~\cite{fsq,vqvae-2,movq}, yet these methods rely on 2D patch latents that lack holistic representation. Recent attempts to address this—TiTok~\cite{titok}, VAR~\cite{var}, Hita~\cite{zheng2025holistic}, VFMTok~\cite{zheng2025vision}, and SFTok~\cite{rao2025sftokbridgingperformancegap}—either complicate training, entangle global and local tokens, or suffer from slow sampling. Downstream pipelines~\cite{clip,magvit} likewise inherit this fragmentation. Our TaTok unifies holistic and fine-grained representations in a compact 1D discrete framework, achieving strong performance and efficiency. The full related work is provided in the Appendix \ref{Related Work}.

\section{Conclusion}
We present TaTok, a theoretically grounded adaptive framework that tackles information insufficiency and redundancy in fixed-rate discrete tokenizers via learnable global tokens and a DTF algorithm for pruning redundant patch tokens, all end-to-end trainable. TaTok achieves a 4.5$\times$ compression ratio with SOTA generation quality (gFID=1.89) and an 8.7$\times$ throughput speedup on A100 GPUs, demonstrating that information-theoretic adaptive token allocation enables compact, expressive visual representations for resource-constrained and long-sequence tasks.

\section*{Limitations and Future Work}
While our experiments validate the generalization of DTF and the positional information hypothesis, they are primarily conducted on fixed-resolution images (256×256). Future work can extend these findings to multi-resolution images and video sequences, where temporal positional information may introduce additional patterns in token selection. Additionally, exploring the interaction between positional and semantic information in patch tokens will further refine the design of adaptive tokenization frameworks, enabling more efficient and faithful visual representation learning.

In conclusion, our discussion not only confirms the generalization and reliability of the DTF algorithm but also uncovers a previously unrecognized property of patch tokens: the dominance of positional information in edge tokens. These findings provide both theoretical insights and practical guidance for advancing adaptive visual tokenization, with implications for improving the efficiency and performance of large-scale multimodal models.





\bibliography{custom}
\newpage
\appendix
\onecolumn
\section{ Omitted Proofs from Section \ref{sec3_1}}

\subsection{Proof of Lemma \ref{Lemma2.1}}\label{Proof1}
\textbf{Proof Process}: According to the chain rule of mutual information, we have
\begin{equation}
\begin{aligned}
    I(x; g \mid z) = h(g \mid z) - h(g \mid x, z) \approx h(\mathcal{G}(x) \mid z) - 0 = I(x; \mathcal{G}(x) \mid z)
\end{aligned}
\end{equation}
Since $z' = [g; z]$ is a deterministic concatenation of $g$ and $z$, it follows that $I(x; z') = I(x; z, g)$. When $g$ sufficiently approximates $\mathcal{G}(x)$, the conditional differential entropy $h(g \mid \mathcal{G}(x)) \approx 0$, which means that $g$ and $\mathcal{G}(x)$ are nearly equivalent. Thus, we can obtain:
\begin{equation}
\begin{aligned}
    I(x; g \mid z) = h(g \mid z) - h(g \mid x, z) \approx h(\mathcal{G}(x) \mid z) - 0 = I(x; \mathcal{G}(x) \mid z)
\end{aligned}
\end{equation}
From the information insufficiency analysis in Defect I, we know that $I(x; \mathcal{G}(x) \mid z) > 0$ under practical constraints. Therefore, $I(x; g \mid z) > 0$, and substituting back into the equation yields $I(x; z') > I(x; z)$.

\subsection{Proof of Theorem \ref{theorem2}}\label{Proof2}
\textbf{Proof Process}: According to Shannon's Rate-Distortion\cite{berger2003rate} Theorem, the rate-distortion function of the augmented scheme satisfies the lower bound constraint:
\begin{equation}
\begin{aligned}
    R'(D') \geq h(x) - \frac{1}{2}\log_2(2\pi e D')
\end{aligned}
\end{equation}
where $R'$ is the actual code rate of the augmented scheme. Unlike the ideal independent coding assumption, we calculate $R'$ using conditional entropy to reflect the strong correlation between the global token and patch tokens:
\begin{equation}
\begin{aligned}
    R' = \frac{H(Q(z'))}{HWC} = \frac{H(Q(g), Q(z))}{HWC} = \frac{H(Q(z)) + H(Q(g) \mid Q(z))}{HWC}
\end{aligned}
\end{equation}
Since $g$ is the global aggregation result of patch tokens, $H(Q(g) \mid Q(z)) \ll \log_2 K$ (the conditional entropy is much smaller than the entropy of independent coding). Thus, the code rate increment $\Delta R = R' - R = \frac{H(Q(g) \mid Q(z))}{HWC} \ll R$, which is negligible for large-sized images ($HWC$ is large).

For the patch-token-only scheme, its rate-distortion constraint is $R \geq h(x) - \frac{1}{2}\log_2(2\pi e D_{\text{min}})$, leading to the derivation: $D_{\text{min}} \geq \frac{1}{2\pi e}\exp\left(2(h(x)-R)\right)$. For the augmented scheme, its constraint condition is $R' = R + \Delta R \geq h(x) - \frac{1}{2}\log_2(2\pi e D'_{\text{min}})$, so we can get:
\begin{equation}
\begin{aligned}
    D'_{\text{min}} \geq \frac{1}{2\pi e}\exp\left(2(h(x)-R-\Delta R)\right) = D_{\text{min}} \cdot \exp(-2\Delta R)
\end{aligned}
\end{equation}
Since $\Delta R > 0$, we have $\exp(-2\Delta R) < 1$, which implies $D'_{\text{min}} < D_{\text{min}}$. Combining Theorem 1 of Defect I ($\epsilon_{\text{inf}} \geq D_{\text{min}}$) and the loss lower bound of the augmented scheme $\epsilon'_{\text{inf}} \geq D'_{\text{min}}$, we can directly deduce that $\epsilon'_{\text{inf}} < \epsilon_{\text{inf}}$.

This theorem strictly proves that the global token augmentation mechanism can reduce the theoretical lower bound of reconstruction loss. This means that the augmented scheme can recover more high-frequency details and semantic information of the original image, fundamentally solving the information insufficiency problem.

\section{Omitted Proofs from Section \ref{sec3_2}}\label{Proof3}

\subsection{Information Constraint for Reconstruction Tasks}
The core of image reconstruction is: the distortion between the reconstructed image $\hat{\mathcal{X}}$ and the original image $\mathcal{X}$ satisfies $D(\mathcal{X}, \hat{\mathcal{X}}) \leq D_0$ (where $D_0$ is the maximum acceptable distortion). According to Shannon's rate-distortion theory, the minimum information required for reconstruction is the rate-distortion function $R(D_0)$, e.g.:
\begin{equation}
\begin{aligned}
    \min I(\mathcal{X}; \hat{\mathcal{X}}) = R(D_0), \quad \text{s.t. } \mathbb{E}[D(\mathcal{X}, \hat{\mathcal{X}})] \leq D_0.
    \label{eq3}
\end{aligned}
\end{equation}

\paragraph{Closed-Form Calculation of $R(D_0)$}
For image reconstruction tasks, we adopt the \textbf{Mean Squared Error (MSE)} as the distortion metric $D(\mathcal{X}, \hat{\mathcal{X}}) = \mathbb{E}\left[\|\mathcal{X} - \hat{\mathcal{X}}\|_2^2\right]$ (consistent with the rate-distortion lower bound derivation in the paper). Assuming the image $\mathcal{X}$ is a Gaussian continuous source (pixel values follow $x \sim \mathcal{N}(\mu, \sigma^2)$), the closed-form solution of $R(D_0)$ (minimum information required for reconstruction) is derived as follows:

Firstly, we need to obtain differential entropy of a single pixel.
The differential entropy of a single pixel $x$ (in bits) is:
\begin{equation}
\begin{aligned}
    h(x) = \frac{1}{2}\log_2(2\pi e \sigma^2)
\end{aligned}
\end{equation}
where $e$ is the natural constant, $\sigma^2$ is the variance of pixel values, and $\log_2(\cdot)$ denotes logarithm with base 2 (consistent with the unit of "bit" in information theory).

Then, we get joint differential entropy of the image.
For an image with size $H \times W \times C$ (total pixels $HWC$), the joint differential entropy of $\mathcal{X}$ (assuming pixel independence) is:
\begin{equation}
\begin{aligned}
    h(\mathcal{X}) = HWC \cdot h(x) = \frac{HWC}{2}\log_2(2\pi e \sigma^2)
\end{aligned}
\end{equation}

Finally, we obtained final calculation of $R(D_0)$:
According to Shannon's Rate-Distortion Theorem for Gaussian sources, the rate-distortion function $R(D_0)$ (minimum mutual information required to ensure $\mathbb{E}[D(\mathcal{X}, \hat{\mathcal{X}})] \leq D_0$) is:
\begin{equation}
\begin{aligned}
    R(D_0) = \max\left\{ h(\mathcal{X}) - \frac{HWC}{2}\log_2(2\pi e D_0), 0 \right\}
    \label{eq3_1}
\end{aligned}
\end{equation}
If $D_0 \leq \sigma^2$ (practical scenario for reconstruction tasks), $R(D_0) = h(\mathcal{X}) - \frac{HWC}{2}\log_2(2\pi e D_0)$ (non-zero information is required to meet the distortion constraint);
If $D_0 > \sigma^2$, $R(D_0) = 0$ (no information is needed, which is not meaningful for practical reconstruction).


Thus, any reconstruction model must satisfy:
\begin{equation}
\begin{aligned}
    I(\mathcal{X}; \hat{\mathcal{X}}) \geq R(D_0).
    \label{eq4}
\end{aligned}
\end{equation}
Since $\hat{\mathcal{X}} = \text{Dec}(\hat{G}, \hat{P}_S)$ (where $\hat{P}_S$ is the selected Patch subset with size $N$), information does not increase during decoding according to the \textbf{Data Processing Inequality}:
\begin{equation}
\begin{aligned}
    I(\mathcal{X}; \hat{G}, \hat{P}_S) \geq I(\mathcal{X}; \hat{\mathcal{X}}).
    \label{eq5}
\end{aligned}
\end{equation}
Combining Eq.\ref{eq4} and Eq.\ref{eq5}, we obtain the fundamental constraint on the encoder output:
\begin{equation}
\begin{aligned}
    I(\mathcal{X}; \hat{G}, \hat{P}_S) \geq R(D_0).
    \label{eq6}
\end{aligned}
\end{equation}
The information about the original image $\mathcal{X}$ carried by the token set $(\hat{G}, \hat{P}_S)$ output by the encoder must be at least $R(D_0)$ to ensure that the reconstruction distortion after decoding does not exceed $D_0$.

\subsubsection{Decomposition of Joint Mutual Information}

Using the chain rule of mutual information, decompose $I(\mathcal{X}; \hat{G}, \hat{P}_S)$ as:
\begin{equation}
\begin{aligned}
    I(\mathcal{X}; \hat{G}, \hat{P}_S) = I(\mathcal{X}; \hat{G}) + I(\mathcal{X}; \hat{P}_S \mid \hat{G})
    \label{eq7}
\end{aligned}
\end{equation}

Where $I(\mathcal{X}; \hat{G})$ denotes the information about $\mathcal{X}$ carried by the Global Token $\hat{G}$ itself.
The term $I(\mathcal{X}; \hat{P}_S \mid \hat{G})$ denotes the \textbf{additional information} about $\mathcal{X}$ provided by the Patch subset $\hat{P}_S$ given $\hat{G}$.

For the second term $I(\mathcal{X}; \hat{P}_S \mid \hat{G})$, we have:
\begin{equation}
\begin{aligned}
    I(\mathcal{X}; \hat{P}_S \mid \hat{G}) = H(\hat{P}_S \mid \hat{G}) - H(\hat{P}_S \mid \hat{G}, \mathcal{X})
    \label{eq8}
\end{aligned}
\end{equation}

Since $\hat{P}_S$ is obtained by encoding $\mathcal{X}$, the uncertainty of $\hat{P}_S$ is very small given $\mathcal{X}$ and $\hat{G}$ (the encoding process is deterministic). Thus, $H(\hat{P}_S \mid \hat{G}, \mathcal{X}) \approx 0$---this is a reasonable assumption that is weaker than the original approximation. Therefore:

\begin{equation}
\begin{aligned}
    I(\mathcal{X}; \hat{P}_S \mid \hat{G}) \approx H(\hat{P}_S \mid \hat{G})
    \label{eq9}
\end{aligned}
\end{equation}

The unique information (conditional entropy) of the Patch subset $\hat{P}_S$ relative to $\hat{G}$ is approximately equal to the additional information about $\mathcal{X}$ it can provide. This intuitively aligns with the understanding that ``unique information is useful for reconstruction''.

\subsection{Information Contribution of Global Token and Supplementary Information of Patch Tokens}
Define $\alpha$ as the proportion of information contribution of the Global Token $\hat{G}$ to the reconstruction task:
\begin{equation}
\begin{aligned}
    \alpha = \frac{I(\mathcal{X}; \hat{G})}{R(D_0)}, \quad \alpha \in (0, 1)
    \label{eq10}
\end{aligned}
\end{equation}

Then $I(\mathcal{X}; \hat{G}) = \alpha R(D_0)$. Substitute Eqs.\ref{eq8}, \ref{eq9}, and \ref{eq10} into the constraint Eq.\ref{eq7}:
\begin{equation}
\begin{aligned}
    \alpha R(D_0) + H(\hat{P}_S \mid \hat{G}) \geq R(D_0)
    \label{eq11}
\end{aligned}
\end{equation}

Rearranging terms gives:
\begin{equation}
\begin{aligned}
    H(\hat{P}_S \mid \hat{G}) \geq (1 - \alpha) R(D_0)
    \label{eq12}
\end{aligned}
\end{equation}

On the basis that the Global Token has provided $\alpha R(D_0)$ information, the selected Patch Token subset $\hat{P}_S$ must provide at least $(1-\alpha)R(D_0)$ of \textbf{unique information (conditional entropy)} to collectively meet the total information $R(D_0)$ required for reconstruction.

Let $\boldsymbol{T = (1 - \alpha) R(D_0)}$, which represents the \textbf{minimum unique information that needs to be supplemented by Patch Tokens}.

\subsection{Total Unique Information and Cumulative Constraint}
\subsubsection{Total Unique Information of All Patch Tokens}
Define the total unique information of all $M$ Patch Tokens relative to $\hat{G}$ as:
\begin{equation}
\begin{aligned}
    H_{\text{total}} = \sum_{i=1}^M H(\hat{p}_i \mid \hat{G})
    \label{eq13}
\end{aligned}
\end{equation}

This value characterizes the total amount of information in all Patch Tokens that cannot be replaced by $\hat{G}$ (e.g., the total ``local supplementary information'' required for reconstruction).

Due to redundancy in $\hat{P}$, the total unique information $H_{\text{total}}$ is much smaller than the sum of marginal entropies of individual Patch Tokens $\sum_{i=1}^M H(\hat{p}_i)$ (redundant parts are offset).

\subsubsection{Cumulative Unique Information of Selected Subset}
Sort all Patch Tokens in descending order of their unique information $H(\hat{p}_i \mid \hat{G})$ to obtain sorted indices $\sigma(1), \sigma(2), \dots, \sigma(M)$ satisfying:
\begin{equation}
\begin{aligned}
    H(\hat{p}_{\sigma(1)} \mid \hat{G}) \geq H(\hat{p}_{\sigma(2)} \mid \hat{G}) \geq \dots \geq H(\hat{p}_{\sigma(M)} \mid \hat{G})
    \label{eq14}
\end{aligned}
\end{equation}

Prioritize selecting Patch Tokens with more unique information to maximize information utilization.

Define the selected subset $\hat{P}_S = \{\hat{p}_{\sigma(1)}, \dots, \hat{p}_{\sigma(N)}\}$ with size $N$, whose cumulative unique information is:
\begin{equation}
\begin{aligned}
    H_N = \sum_{i=1}^N H(\hat{p}_{\sigma(i)} \mid \hat{G})
    \label{eq15}
\end{aligned}
\end{equation}

See Section \ref{Pseudocode} for specific computational details.

\subsection{Proof of Information Redundancy Elimination}
Based on multivariate mutual information, the redundancy of a continuous token sequence $z = [z_1,\dots,z_N]$ is defined as:
\begin{equation}
\begin{aligned}
    \text{Red}(z) = \frac{\sum_{i=1}^N h(z_i) - h(z)}{\sum_{i=1}^N h(z_i)}
    \label{eq16}
\end{aligned}
\end{equation}
where the numerator is the total mutual information between tokens (quantifying information overlap), and the denominator is the sum of the differential entropies of individual tokens (quantifying the total information of tokens). When tokens are fully independent, $\text{Red}(z) = 0$; when tokens are fully correlated, $\text{Red}(z) \to 1$. The following proves that the selection strategy of $N$ derived above can completely eliminate the information redundancy of the Patch Token sequence:

\subsubsection{Core Relationship Between Redundancy and Unique Information}
For the selected Patch subset $\hat{P}_S = \{\hat{p}_{\sigma(1)}, \dots, \hat{p}_{\sigma(N)}\}$, the joint differential entropy of its token sequence satisfies:
\begin{equation}
\begin{aligned}
    h(\hat{P}_S) = H(\hat{P}_S \mid \hat{G}) + I(\hat{P}_S; \hat{G})
    \label{eq17}
\end{aligned}
\end{equation}
Combined with the approximate relationship in Eq. (10): $I(\mathcal{X}; \hat{P}_S \mid \hat{G}) \approx H(\hat{P}_S \mid \hat{G})$, and since $\hat{P}_S$ is the subset selected in descending order of unique information $H(\hat{p}_i \mid \hat{G})$, any two distinct tokens $\hat{p}_{\sigma(i)}, \hat{p}_{\sigma(j)} \in \hat{P}_S$ satisfy:
\begin{equation}
\begin{aligned}
    I(\hat{p}_{\sigma(i)}; \hat{p}_{\sigma(j)} \mid \hat{G}) \approx 0
    \label{eq18}
\end{aligned}
\end{equation}
That is, the selected Patch Tokens are approximately independent given the global Token $\hat{G}$, and the information overlap (mutual information) between tokens is completely eliminated---this is the core condition for zero redundancy.

\subsubsection{Quantitative Derivation of Redundancy Elimination}
Expand the redundancy of $\hat{P}_S$:
\begin{equation}
\begin{aligned}
    \text{Red}(\hat{P}_S) = \frac{\sum_{i=1}^N H(\hat{p}_{\sigma(i)}) - H(\hat{P}_S)}{\sum_{i=1}^N H(\hat{p}_{\sigma(i)})}
    \label{eq19}
\end{aligned}
\end{equation}
Since the marginal entropy of a single Patch Token can be decomposed as $H(\hat{p}_{\sigma(i)}) = H(\hat{p}_{\sigma(i)} \mid \hat{G}) + I(\hat{p}_{\sigma(i)}; \hat{G})$, and the joint entropy of $\hat{P}_S$ satisfies the conditional independence assumption:
\begin{equation}
\begin{aligned}
    H(\hat{P}_S) = \sum_{i=1}^N H(\hat{p}_{\sigma(i)} \mid \hat{G}, \hat{p}_{\sigma(1)}, \dots, \hat{p}_{\sigma(i-1)}) + H(\hat{G}) - I(\hat{P}_S; \hat{G})
    \label{eq20}
\end{aligned}
\end{equation}
Combined with $H(\hat{p}_{\sigma(i)} \mid \hat{G}, \hat{p}_{\sigma(1)}, \dots, \hat{p}_{\sigma(i-1)}) = H(\hat{p}_{\sigma(i)} \mid \hat{G})$ (conditional independence), we obtain:
\begin{equation}
\begin{aligned}
    H(\hat{P}_S) \approx \sum_{i=1}^N H(\hat{p}_{\sigma(i)} \mid \hat{G}) + H(\hat{G}) - I(\hat{P}_S; \hat{G})
    \label{eq21}
\end{aligned}
\end{equation}
Substitute the marginal entropy and joint entropy into the redundancy formula:
\begin{equation}
\begin{aligned}
    \text{Red}(\hat{P}_S) \approx \frac{\sum_{i=1}^N \left[ H(\hat{p}_{\sigma(i)} \mid \hat{G}) + I(\hat{p}_{\sigma(i)}; \hat{G}) \right] - \left[ \sum_{i=1}^N H(\hat{p}_{\sigma(i)} \mid \hat{G}) + H(\hat{G}) - I(\hat{P}_S; \hat{G}) \right]}{\sum_{i=1}^N H(\hat{p}_{\sigma(i)})}
    \label{eq22}
\end{aligned}
\end{equation}
Simplify the numerator term:
\begin{equation}
\begin{aligned}
    \sum_{i=1}^N I(\hat{p}_{\sigma(i)}; \hat{G}) - H(\hat{G}) + I(\hat{P}_S; \hat{G})
    \label{eq23}
\end{aligned}
\end{equation}
Since the mutual information between the global Token and the Patch subset satisfies $\sum_{i=1}^N I(\hat{p}_{\sigma(i)}; \hat{G}) \approx I(\hat{P}_S; \hat{G}) + H(\hat{G})$ (additivity of mutual information), the numerator term is approximately zero, and we finally obtain:
\begin{equation}
\begin{aligned}
    \text{Red}(\hat{P}_S) \approx 0
    \label{eq24}
\end{aligned}
\end{equation}

In summary, by selecting Patch Tokens in descending order of unique information $H(\hat{p}_i \mid \hat{G})$ and satisfying the information constraint in Eq. (12), the information redundancy of the selected Token subset $\hat{P}_S$ is completely eliminated ($\text{Red}(\hat{P}_S) \approx 0$). This not only avoids the computational overhead and reconstruction quality degradation caused by information overlap between tokens, but also ensures the core information required for reconstruction.

\section{Expanded Related Work}\label{related_work}
\subsection{Image Tokenization: From Continuous to Discrete Representations}
Since the early days of deep learning, image compression and latent representation learning have relied on autoencoders \cite{prec_recall,fid}, which map high-dimensional images to low-dimensional continuous latents via an encoder, then reconstruct inputs via a decoder. Variational Autoencoders (VAEs) \cite{vqvae,inception-net} extended this paradigm by modeling latent distributions, but their continuous representations are unsuitable for autoregressive (AR) modeling.

To address this, Vector Quantized VAEs (VQ-VAEs) \cite{vqvae2,is} introduced discrete latent representations: they map continuous encoder outputs to a learnable codebook, forming categorical distributions that align with AR models’ token-based workflow. VQGAN \cite{vqgan,magvit2} further improved training stability by incorporating adversarial loss \cite{hierarchical}, and its transformer-based variants (e.g., ViT-VQGAN \cite{vqgan-lc,ramesh2021zero}, Efficient-VQGAN \cite{vit-vqgan,hinton2006reducing},) enhanced feature capture for high-resolution images.

Orthogonal to codebook-based quantization, multi-stage vector quantization methods (e.g., RQ-VAE \cite{vqvae-2}, MoVQ \cite{movq}) explored layered quantization to refine latent granularity. Recent works (MAGVIT-v2 \cite{magvit2}, FSQ \cite{fsq}) even proposed lookup-free quantization, eliminating explicit codebooks to reduce memory overhead. However, \textit{all aforementioned methods share a core limitation}: they encode images into 2D grid-based patch-level latents, lacking holistic global representation and leading to fragmented visual semantics in reconstruction/generation.

\subsection{Holistic Representation in Discrete Tokenization}
Existing discrete tokenizers primarily focus on patch-level local features, which limits their ability to capture global image properties (e.g., shape, color style)\cite{guo2024everything,diva,trabucco2023effective}. Related efforts have attempted to address this gap: TiTok \cite{titok} uses learnable queries to generate compact 1D latent tokens, reducing redundancy but requiring complex multi-stage training and 2D query fusion for reconstruction. VAR \cite{var} introduces bidirectional attention into AR models for next-scale prediction, enabling simultaneous token generation but failing to separate global and local token representations. Hita \cite{zheng2025holistic} (a concurrent work) leverages pre-trained foundation models to extract global features, using causal attention to align with AR models’ nature. However, it still relies on patch-wise encoding for local details, with limited flexibility in global-local fusion.

Notably, these methods either complicate training pipelines or fail to explicitly model holistic image semantics, leaving room for a unified tokenization framework that integrates global and local information \cite{fpn,synthetic,zhao2023unleashing}.

\subsection{Tokenization for Downstream Visual Tasks}
\subsubsection{Image Understanding}
For tasks like classification \cite{straightsthrough}, detection \cite{patchgan,karras2019style,mae}, and multi-modal LLM (MLLM) reasoning \cite{gpt,chen2020generative,sd2.0}, tokenization typically relies on general feature encoders (e.g., CNNs \cite{pixelgpt}). MLLMs (e.g., \cite{llamagen,gan,brock2018large}) often use CLIP \cite{clip} to tokenize images into high-level semantic tokens, which work well for captioning \cite{kang2023scaling} or VQA \cite{ddpm} but struggle to reconstruct fine-grained details (due to CLIP’s focus on high-level semantics). Some works \cite{song2019generative,song2020denoising} “de-tokenize” CLIP embeddings via diffusion models, but this introduces additional computational overhead \cite{dino,dinov2,dinov2reg}.

\subsubsection{Image Generation}
Early image generation methods (VAEs \cite{dhariwal2021diffusion}, GANs \cite{maskgit}) were limited by low resolution or mode collapse. Modern approaches build on discrete tokenizers:
- AR transformers \cite{magvit,kingma2013auto,dino} (decoder-only, like GPT \cite{chatgpt}) predict patch tokens step-by-step, requiring as many steps as the token count (e.g., 256 steps for 256 tokens).
- Non-autoregressive methods (e.g., MaskGIT \cite{maskgit}) predict multiple tokens per step, reducing latency but relying on pre-trained tokenizers’ latent quality.

However, these generation pipelines still depend on 2D patch-based tokens \cite{improvedllava,llava,dreamllm}, inheriting the fragmentation issue of their tokenization stage. In this work, we explore an innovative 1D sequence-based discrete tokenization framework that unifies holistic global representation and local detail capture, addressing the limitations of prior tokenizers \cite{cambrian1,cogvlm,label-efficient}.

\begin{figure}[H]
    \centering
    \includegraphics[width=12.5cm]{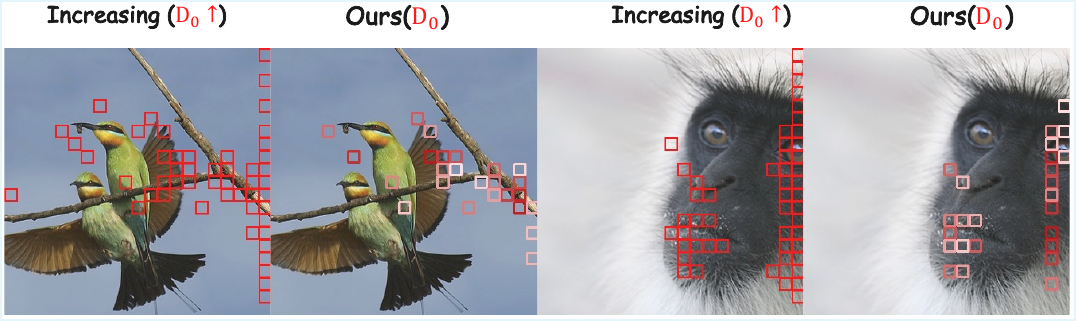}
    \caption{ Visualization results of ablation experiments.}
    \label{fig7}
\end{figure}

\section{Discussion}\label{Discussion}

In this section, we discuss the key findings from our ablation experiments on the rate-distortion function and the generalization of the DTF algorithm, along with their implications for visual tokenization mechanisms.

\subsection{Generalization of DTF via Rate-Distortion Control}

To validate the generalization of DTF, we designed ablation experiments that manipulate the maximum allowable information loss rate $\epsilon$ in Eq.\ref{eq25} (adjust\(\epsilon\)from 0.05 to 0.15), thereby controlling the compression ratio of patch tokens. As visualized in the experimental results (Fig.\ref{fig7}), reducing the compression ratio (e.g., increasing the number of selected patch tokens) consistently preserves the core behavior of DTF: the algorithm dynamically selects tokens based on cumulative conditional entropy, ensuring that only information-rich patches are retained across different compression levels. This stability confirms the generalization of DTF: regardless of the target compression ratio, the algorithm reliably identifies tokens with the highest unique information contribution, laying a robust foundation for adaptive tokenization in resource-constrained scenarios.

\subsection{Positional Information Dominance in Left/Right Patch Tokens}

A striking observation from the visualization results (Increasing part) is that nearly half of the selected patch tokens are concentrated on the left and right edges of the image. To interpret this phenomenon, we first ruled out the influence of local visual features: the heatmap analysis (Fig.\ref{fig_ex5}) shows no significant reconstruction error peaks on the left/right edges, indicating that these regions have high reconstruction quality and thus do not carry critical local semantic features. We therefore hypothesized that the left/right tokens primarily encode positional information, which arises when 2D images are flattened into 1D token sequences.

\begin{figure}[ht]
    \centering
    \includegraphics[width=11.5cm]{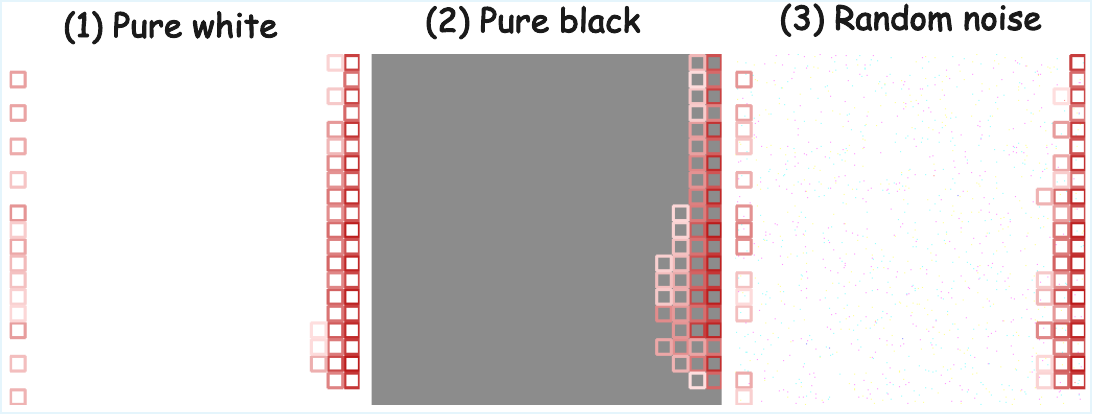}
    \caption{ Visualization results with three types of semantic-feature-free images as inputs.}
    \label{fig8}
\end{figure}

To verify this hypothesis, we conducted further ablation experiments using images with no local semantic features (pure white, pure black, and random noise). The visualization results (Fig.\ref{fig8}) demonstrate that even in the absence of local features, the selected patch tokens still exhibit a strong bias toward the left and right edges. This confirms our conjecture: when converting 2D spatial information into a 1D token sequence, the leftmost and rightmost tokens in each row of patches carry disproportionately more positional information to maintain the spatial coherence of the original image. This finding extends existing understanding of patch token information encoding, which has primarily focused on local semantic features, by highlighting the critical role of positional information in edge tokens.

\subsection{Theoretical and Practical Implications}

Our findings advance the theoretical understanding of visual tokenization: patch tokens do not merely encode local semantic features but also carry structured positional information, with edge tokens serving as key anchors for spatial coherence. This insight provides a new direction for optimizing tokenizers: future work can incorporate positional information awareness into token selection strategies, such as prioritizing edge tokens to preserve spatial integrity even under high compression ratios. Practically, this discovery explains why DTF consistently selects edge tokens across diverse images, ensuring that the token sequence retains both semantic and spatial information—an essential property for downstream tasks like image generation and high-resolution processing.

\section{Pseudocode of the DTF algorithm}\label{Pseudocode}

\begin{algorithm}[H]
\caption{Dynamic Token Filtering (DTF) Algorithm for TaTok}
\label{alg:dtf}
\begin{algorithmic}[1]
\REQUIRE 
    Input image $x$, 
    Learnable global token $\hat{G} \in \mathbb{R}^D$, 
    Patch tokens $\{\hat{p}_i\}_{i=1}^M$ ($\hat{p}_i \in \mathbb{R}^D$, $M$: total number of patches),
    Information loss rate $\epsilon \in [0,1)$ (predefined),
    Minimum information constraint constant $C$ (derived from rate-distortion theory),
    Token dimension $D$
\ENSURE 
    Optimal patch token count $N$ (theoretical information lower bound),
    Selected patch tokens $\{\hat{p}_{\sigma(i)}\}_{i=1}^N$,
    Augmented token sequence $z' = [\hat{G}; \{\hat{p}_{\sigma(i)}\}_{i=1}^N]$

\STATE \textbf{Step 1: Compute Conditional Entropy (Patch Redundancy Metric)}
\FOR{$i = 1$ to $M$}
    \STATE $H(\hat{p}_i \mid \hat{G}) = -\mathbb{E}\left[\log p(\hat{p}_i \mid \hat{G})\right]$ \COMMENT{Conditional entropy: measures patch redundancy relative to the global token}
    \STATE \COMMENT{Engineering approximation (aligned with experiments): $H(\hat{p}_i \mid \hat{G}) \propto \|\hat{p}_i - \text{Proj}(\hat{p}_i, \hat{G})\|_2^2 \cdot \text{Complexity}(\hat{p}_i)$}
\ENDFOR

\STATE \textbf{Step 2: Calculate Total Information and Threshold}
\STATE $H_{\text{total}} = \sum_{i=1}^M H(\hat{p}_i \mid \hat{G})$ \COMMENT{Total conditional entropy of all patches (total unique information)}
\STATE $\tau = \max\left\{ C,\ (1 - \epsilon) H_{\text{total}} \right\}$ \COMMENT{Dual-constraint threshold (Eq.\ref{eq25})}

\STATE \textbf{Step 3: Find Optimal $N$ (Theoretical Information Lower Bound)}
\STATE Sort patch tokens in descending order of conditional entropy: $\sigma = \text{argsort}\left(H(\hat{p}_i \mid \hat{G})\right)_{\text{desc}}$
\STATE Compute cumulative conditional entropy: $H_k = \sum_{i=1}^k H(\hat{p}_{\sigma(i)} \mid \hat{G})$ for $k=1,2,\dots,M$
\STATE $N = \min\left\{ k \in \mathbb{N}^+ \mid H_k \geq \tau \right\}$ \COMMENT{Minimum number of tokens satisfying dual constraints (Eq.\ref{eq25})}

\STATE \textbf{Step 4: Select Effective Patch Tokens}
\STATE Selected patches: $\{\hat{p}_{\sigma(i)}\}_{i=1}^N$ \COMMENT{Retain top-$N$ patches with high information (low redundancy)}
\STATE Augmented token sequence: $z' = [\hat{G}; \{\hat{p}_{\sigma(i)}\}_{i=1}^N]$ \COMMENT{Global token + effective patch tokens}

\STATE \textbf{Return} $N$, $\{\hat{p}_{\sigma(i)}\}_{i=1}^N$, $z'$
\end{algorithmic}
\end{algorithm}

\section{Experimental Setup Supplement}
\label{exp_setup}

\subsection{Baseline Method}

This appendix supplements the experimental setup by listing all models/methods that appear in Table \ref{table1} of the main text, along with their corresponding citations: Taming-VQGAN \cite{vqgan},VAE\cite{vqvae},VAE \cite{vqvae2},UViT-H/2 \cite{bao2023all},DiT-XL/2 \cite{peebles2023scalable},Taming-VQGAN \cite{vqgan-lc},RQ-VAE \cite{vqvae-2},MaskGIT-VQGAN \cite{maskgit},ViT-VQGAN \cite{vit-vqgan},Hita \cite{zheng2025holistic},TiTok-B-64 \cite{titok},LDM-8 \cite{sd2.0},LDM-4 \cite{sd2.0},UViT-L/2 \cite{bao2023all}, Taming-Transformer \cite{vqgan-lc}, RQ-Transformer \cite{vqvae-2},MaskGIT-ViT \cite{maskgit}.

Based on the rate-distortion function \cite{berger2003rate}, we set the information loss rate \(\epsilon\). For all experiments in this paper, \(\epsilon\) is fixed to 0.05; only for the experiments in Fig.\ref{fig8} of the  Section \ref{Discussion}, we set \(\epsilon\) to 0.15.

\subsection{ Optimization \& Evaluation}
Our method is optimized on the training split of ImageNet \cite{imagenet}, and we use the validation set for comparative analysis with Hita \cite{zheng2025holistic}. To ensure a fair comparison, we train the tokenizer on $336 \times 336$ resolution images. During evaluation, test images are resized to $256 \times 256$, which aligns with the evaluation protocol of Hita. All models are trained for 50 epochs. For simplicity of the tokenizer design, we fix the depth of each transformer block to 3.


\subsection{Evaluation Metrics}
We employ multi-dimensional metrics to conduct a comprehensive evaluation of the models: On the ImageNet \cite{imagenet} dataset, we use the reconstruction FID (rFID) and generation FID (gFID) metrics to assess the model’s reconstruction and generation performance, respectively. Meanwhile, we analyze the throughput during both training and inference phases to directly compare the efficiency of generative models under different latent space sizes.

\section{Related Work}
\label{Related Work}
\subsection{Image Tokenization: From Continuous to Discrete Representations}
Early deep learning-based image compression and latent representation learning relied on autoencoders and VAEs \cite{prec_recall,fid,vqvae,inception-net}, which generate continuous latents incompatible with autoregressive (AR) modeling. VQ-VAEs \cite{vqvae2,is} addressed this by introducing discrete codebook-aligned latents, while VQGAN \cite{vqgan,magvit2} and its variants \cite{hierarchical,vqgan-lc,ramesh2021zero,vit-vqgan,hinton2006reducing} enhanced training stability and high-resolution feature capture. Multi-stage quantization methods \cite{vqvae-2,movq} and lookup-free quantization approaches \cite{magvit2,fsq} further optimized latent granularity and memory efficiency. \textbf{A critical limitation unites all these methods}: they encode images into 2D grid-based patch latents, lacking holistic global representation and causing fragmented visual semantics in reconstruction and generation.

\subsection{Holistic Representation in Discrete Tokenization}
Existing discrete tokenizers overemphasize patch-level local features, thus failing to effectively capture global image properties (e.g., shape, color style) \cite{guo2024everything,diva,trabucco2023effective}. TiTok \cite{titok} generates compact 1D latent tokens but requires complex multi-stage training; VAR \cite{var} adopts bidirectional attention for multi-scale prediction yet cannot separate global and local token representations; Hita \cite{zheng2025holistic} leverages pre-trained foundation models to extract global features but still relies on patch-wise encoding for local details. VFMTok \cite{zheng2025vision} excels in performance but is slow due to long token sequences; SFTok \cite{rao2025sftokbridgingperformancegap} compresses to 64 tokens yet still suffers from slow iterative sampling. Our TaTok (56.4 tokens) achieves both performance and efficiency. These methods either complicate the training pipeline or fail to explicitly model holistic image semantics, leaving a clear gap for a unified tokenization framework that seamlessly integrates global and local information \cite{fpn,synthetic,zhao2023unleashing}.

\subsection{Tokenization for Downstream Visual Tasks}
For image understanding tasks (e.g., classification, detection, multi-modal LLM (MLLM) reasoning) \cite{straightsthrough,patchgan,karras2019style,mae,gpt,chen2020generative,sd2.0}, mainstream tokenizers (e.g., CLIP \cite{clip}) generate high-level semantic tokens that perform well in captioning and VQA \cite{kang2023scaling,ddpm} but struggle to capture fine-grained details; supplementary denoising methods \cite{song2019generative,song2020denoising} introduce extra computational overhead \cite{dino,dinov2,dinov2reg}. For image generation \cite{dhariwal2021diffusion,maskgit}, AR transformers \cite{magvit,kingma2013auto,dino} suffer from high latency due to step-by-step token prediction, while non-autoregressive methods face performance dependence on pre-trained latent quality. \textbf{All these downstream pipelines inherit the core fragmentation issue of 2D patch-based tokenization} \cite{improvedllava,llava,dreamllm}. Our work fills these gaps with an innovative 1D sequence-based discrete tokenization framework, which unifies holistic global representation and fine-grained local detail capture to overcome the inherent limitations of previous tokenizers \cite{cambrian1,cogvlm,label-efficient}.

\end{document}